\documentclass{article}
\usepackage[final,activate={true,nocompatibility}]{microtype}
\usepackage{graphicx}
\usepackage{subfigure}
\usepackage[USenglish]{babel}

\usepackage{amsmath}
\usepackage{amssymb}
\usepackage{amsthm}
\theoremstyle{definition}
\newtheorem{definition}{Definition}[section]
\usepackage{thmtools}
\usepackage{thm-restate}

\usepackage{bbm}
\usepackage{optidef}
\usepackage{mathtools}
\usepackage{pifont}
\usepackage{nicefrac}

\usepackage{array}
\usepackage{adjustbox}
\usepackage{multirow}
\usepackage{booktabs}
\usepackage{balance}

\usepackage[firstpage]{draftwatermark}
\SetWatermarkText{
 \hspace*{3.5in}
 \raisebox{10in}{
  \includegraphics[height=0.75in]{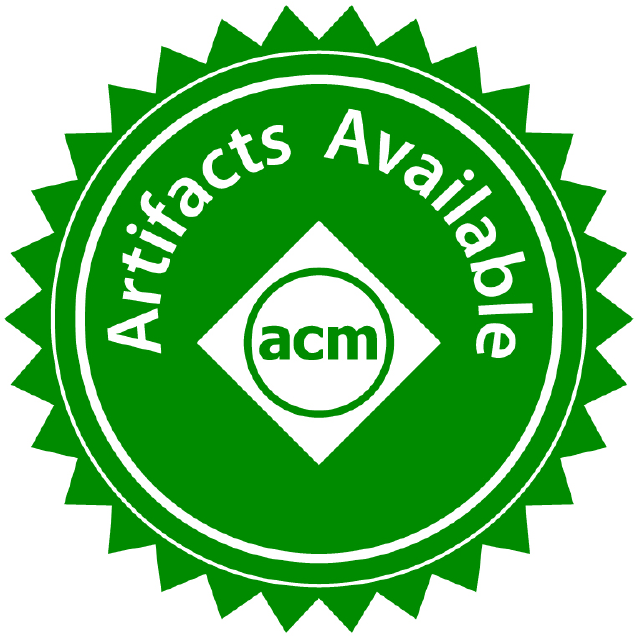}\hspace{-30pt}
  \includegraphics[height=0.75in]{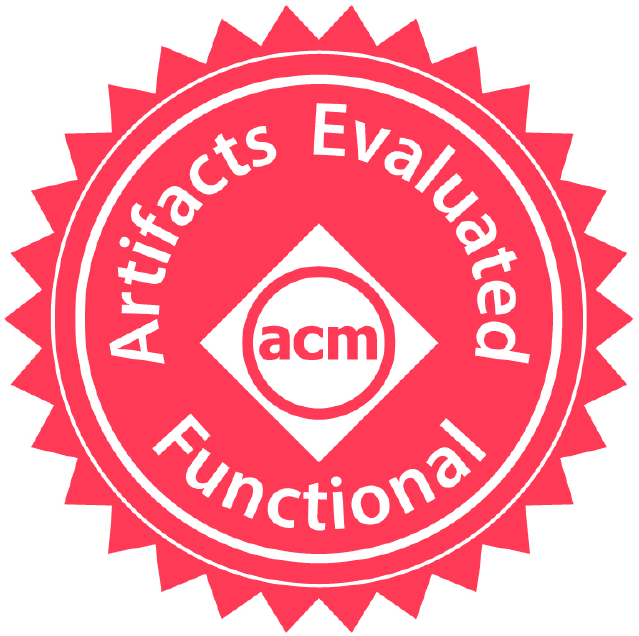}\hspace{-30pt}
  \includegraphics[height=0.75in]{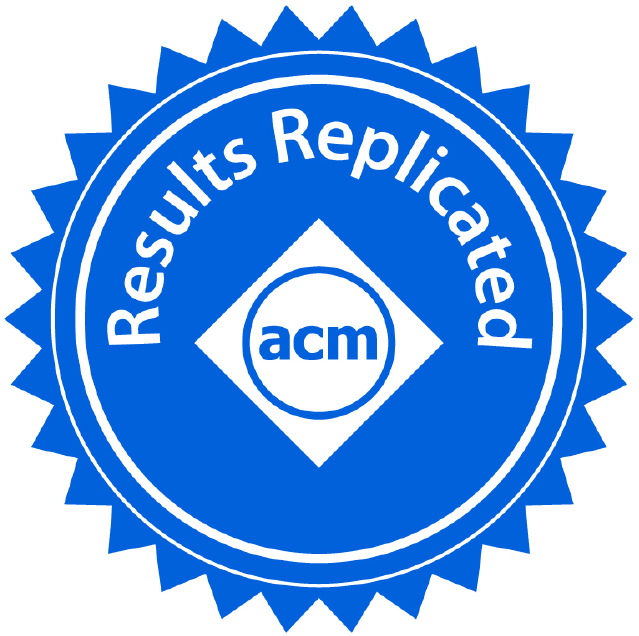}
 }
}
\SetWatermarkAngle{0}

\usepackage{hyperref}

\usepackage[accepted]{mlsys2020}

\mlsystitlerunning{Checkmate: Breaking the Memory Wall with Optimal Tensor Rematerialization}
\newcommand{\OURS}{Checkmate}

\usepackage{xcolor}

\newcommand{\GC}[2]{\text{mem\_freed}_{#2}(#1)}
\newcommand{\NumHazards}[1]{\text{num\_hazards}(#1)}
\newcommand{\Deps}[1]{\textsc{Deps}[#1]}
\newcommand{\Users}[1]{\textsc{Users}[#1]}

\newcommand{\Collectable}[1]{\textsc{Free}_{#1}}
\newcommand{\Min}{M_\text{input}}
\newcommand{\Mparam}{M_\text{param}}
\newcommand{\Mbudget}{M_\text{budget}}

\newcommand{\com}[1]{s_{#1}}
\newcommand{\Regs}[1]{\textsc{Regs}[#1]}
\newcommand{\eg}{\textit{e.g.}}
\newcommand{\ie}{\textit{i.e.}}

\begin{document}
\twocolumn[
	\mlsystitle{\OURS{}: Breaking the Memory Wall\\ with Optimal Tensor Rematerialization}
	\mlsyssetsymbol{equal}{*}
	\begin{mlsysauthorlist}
		\mlsysauthor{Paras Jain}{equal,cal}
		\mlsysauthor{Ajay Jain}{equal,cal}
		\mlsysauthor{Aniruddha Nrusimha}{cal}\\
		\mlsysauthor{Amir Gholami}{cal}
		\mlsysauthor{Pieter Abbeel}{cal}
		\mlsysauthor{Kurt Keutzer}{cal}
		\mlsysauthor{Ion Stoica}{cal}
		\mlsysauthor{Joseph E. Gonzalez}{cal}
	\end{mlsysauthorlist}
	\mlsysaffiliation{cal}{Department of EECS, UC Berkeley}
	\mlsyscorrespondingauthor{Paras Jain}{parasj@berkeley.edu}
	\mlsyskeywords{Machine Learning, MLSys, Deep Learning, Compiler, Python, TensorFlow, Memory, GPU}
	\vskip 0.3in
	\begin{abstract}
We formalize the problem of trading-off DNN training time and memory requirements as the \textit{tensor rematerialization} optimization problem, a generalization of prior checkpointing strategies. We introduce \OURS{}, a system that solves for optimal rematerialization schedules in reasonable times (under an hour) using off-the-shelf MILP solvers or near-optimal schedules with an approximation algorithm, then uses these schedules to accelerate millions of training iterations. Our method scales to complex, realistic architectures and is hardware-aware through the use of accelerator-specific, profile-based cost models. In addition to reducing training cost, \OURS{} enables real-world networks to be trained with up to 5.1$\times$ larger input sizes. Checkmate is an open-source project, available at \url{https://github.com/parasj/checkmate}.
	\end{abstract}
]
\printAffiliationsAndNotice{\mlsysEqualContribution}

\section{Introduction}
\label{sec:introduction}
Deep learning training workloads demand large amounts of high bandwidth memory. Researchers are pushing the memory capacity limits of hardware accelerators such as GPUs by training neural networks on high-resolution images~\cite{dong_superres,kim_superres,tai_superres}, 3D point-clouds~\cite{chenMultiview3DObject2017, yangHDNETExploitingHD}, and long natural language sequences~\cite{vaswani_attention_2017,devlin_bert:_2018,child_generating_2019}.
In these applications, training memory usage is dominated by the intermediate activation tensors needed for backpropagation (Figure~\ref{fig:mem_breakdown}).

The limited availability of high bandwidth on-device memory creates a \textit{memory wall} that stifles exploration of novel architectures.
Across applications, authors of state-of-the-art models cite memory as a limiting factor in deep neural network (DNN) design \cite{krizhevsky_imagenet_2012,he_deep_2016, chen_deeplab:_2016,gomez_reversible_2017,pohlen2017FRRN,child_generating_2019,liu_roberta:_2019,dai_transformer-xl:_2019}.

As there is insufficient RAM to cache all activation tensors for backpropagation, some select tensors can be discarded during forward evaluation. When a discarded tensor is necessary as a dependency for gradient calculation, the tensor can be \textit{rematerialized}. As illustrated in Figure~\ref{fig:illustrated_example}, rematerializing values allows a large DNN to fit within memory at the expense of additional computation.

\begin{figure}[t]
	\centering
	\includegraphics[width=\columnwidth]{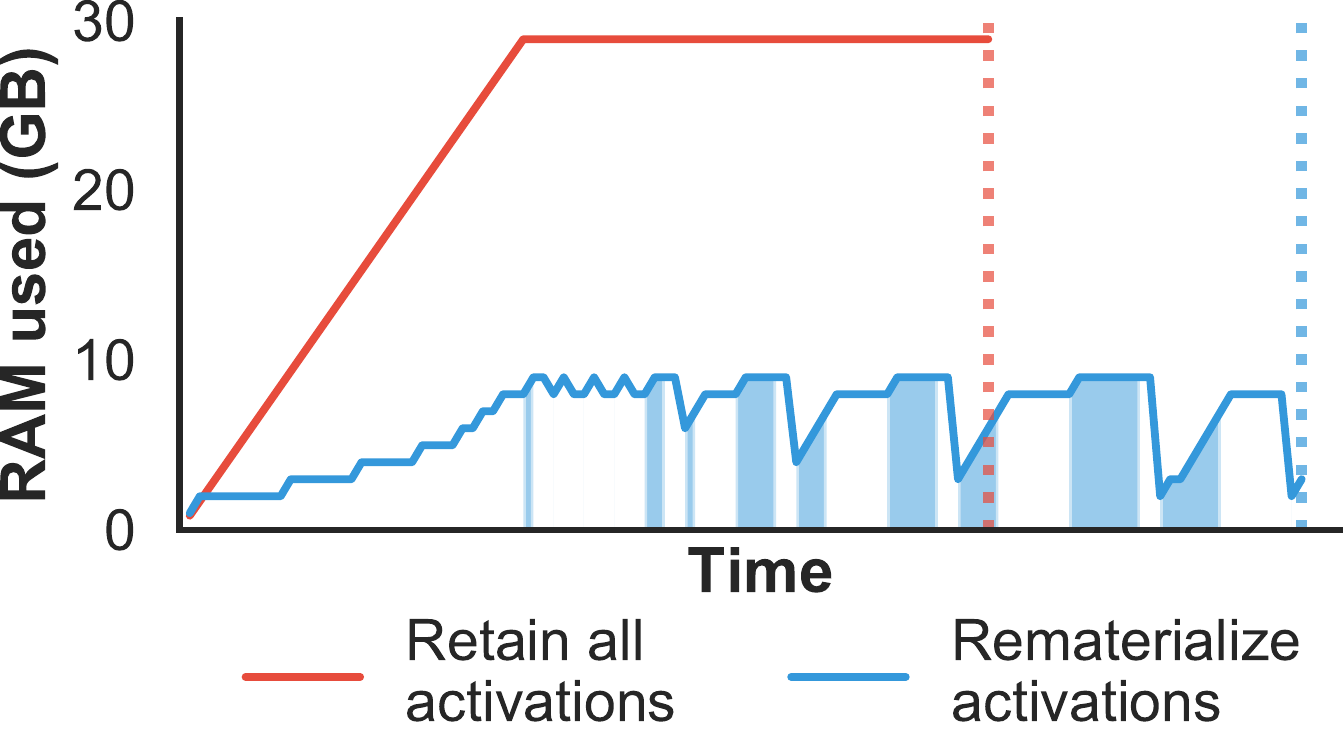}
	\vspace{-1em}
	\caption{This 32-layer deep neural network requires 30GB of memory during training in order to cache forward pass activations for the backward pass. Freeing certain activations early and rematerializing them later reduces memory requirements by 21GB at the cost of a modest runtime increase. Rematerialized layers are denoted as shaded blue regions. We present Checkmate, a system to rematerialize large neural networks \emph{optimally}. Checkmate is hardware-aware, memory-aware and supports arbitrary DAGs.}
	\label{fig:illustrated_example}
	\vspace{-1.5em}
\end{figure}

\citet{griewank_algorithm_2000} and \citet{chen_training_2016} present heuristics for rematerialization when the forward pass forms a linear graph, or path graph. They refer to the problem as checkpointing. 
However, their approaches cannot be applied generally to nonlinear DNN structures such as residual connections, and rely on the strong assumption that all nodes in the graph have the same cost. Prior work also assumes that gradients may never be rematerialized. These assumptions limit the efficiency and generality of prior approaches.

Our work formalizes tensor rematerialization as a constrained optimization problem. 
Using off-the-shelf numerical solvers, we are able to discover optimal rematerialization strategies for arbitrary deep neural networks in TensorFlow with non-uniform computation and memory costs.  
We demonstrate that optimal rematerialization allows larger batch sizes and substantially reduced memory usage with minimal computational overhead across a range of image classification and semantic segmentation architectures.
As a consequence, our approach allows researchers to easily explore larger models, at larger batch sizes, on more complex signals with minimal computation overhead. 

\begin{figure}[t]
	\centering
	\includegraphics[width=\linewidth]{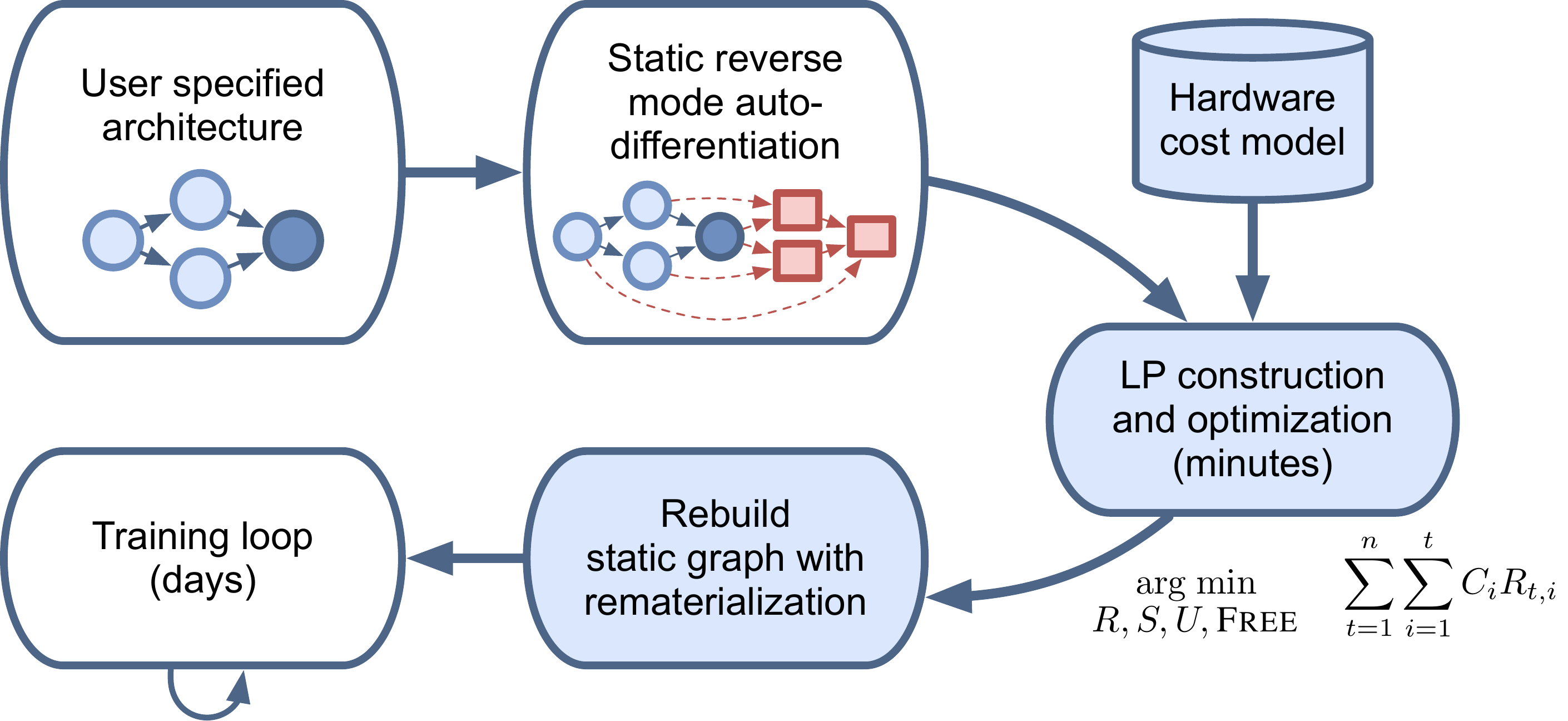}
	\vspace{-0.5em}
	\caption{Overview of the \OURS{} system.}
	\label{fig:system}
\end{figure}

In particular, the contributions of this work include: 
\begin{itemize}
	\itemsep0em
	\item a formalization of the rematerialization problem as a mixed integer linear program with a substantially more flexible search space than prior work, in Section~\ref{sec:ilp:complete_ilp}.
    \item a fast approximation algorithm based on two-phase deterministic LP rounding, in Section~\ref{sec:approx}.
	\item \OURS{}, a system implemented in TensorFlow that enables training models with up to 5.1$\times$ larger input sizes than prior art at minimal overhead.
\end{itemize}

\section{Motivation}
\label{sec:motivation}

While inference optimizations are well studied \cite{jain2019ooo}, training workloads have received less attention. Memory consumption during training consists of (a) intermediate features, or activations, whose size depends on input dimensions and (b) parameters and their gradients whose size depends on weight dimensions. Given that inputs are often several order of magnitude larger than kernels, most memory is used by features, demonstrated in Figure~\ref{fig:mem_breakdown}.

Frameworks such as TensorFlow \cite{abadi_tensorflow:_2016} and PyTorch \cite{paszke_automatic_2017, paszke_pytorch_2019} store all activations during the forward pass. 
Gradients are backpropagated from the loss node, and each activation is freed after its gradient has been calculated. In Figure~\ref{fig:illustrated_example}, we compare this memory intensive policy and a rematerialization strategy for a real neural network. Memory usage is significantly reduced by deallocating some activations in the forward pass and recomputing them in the backward pass.
Our goal is fit an arbitrary network within our memory budget while incurring the minimal additional runtime penalty from recomputation.

Most prior work assumes networks have linear graphs. For example, \citet{chen_training_2016} divides the computation into $\sqrt{n}$ segments, each with $\sqrt{n}$ nodes. Each segment endpoint is stored during the forward pass. During the backward pass, segments are recomputed in reverse order at $O(n)$ cost. 

Linear graph assumptions limit applicability of prior work.
For example, while the popular ResNet50~\cite{he_deep_2016} can be linearized by treating each residual block as a single node, this leads to inefficient solutions. 
For networks with longer skip connections, \eg{}, U-Net~\cite{ronneberger_u-net:_2015}, grouping nodes oversimplifies the graph.

Prior work also assumes all layers are equally expensive to recompute.
In the VGG19 architecture~\cite{simonyan_very_2014}, the largest layer is \emph{six orders of magnitude} more expensive than the smallest layer.

Our work makes few assumptions on neural network graphs. We explore a solution space that allows for (a) arbitrary graphs with several inputs and outputs for each node, (b) variable memory costs across layers and (c) variable computation costs for each layer (such as FLOPs or profiled runtimes). We constrain solutions to simply be \emph{correct} (a node's dependencies must be materialized before it can be evaluated) and \emph{within the RAM budget} (at any point during execution, resident tensors must fit into RAM).

\begin{figure}[t]
	\centering
	\includegraphics[width=\columnwidth]{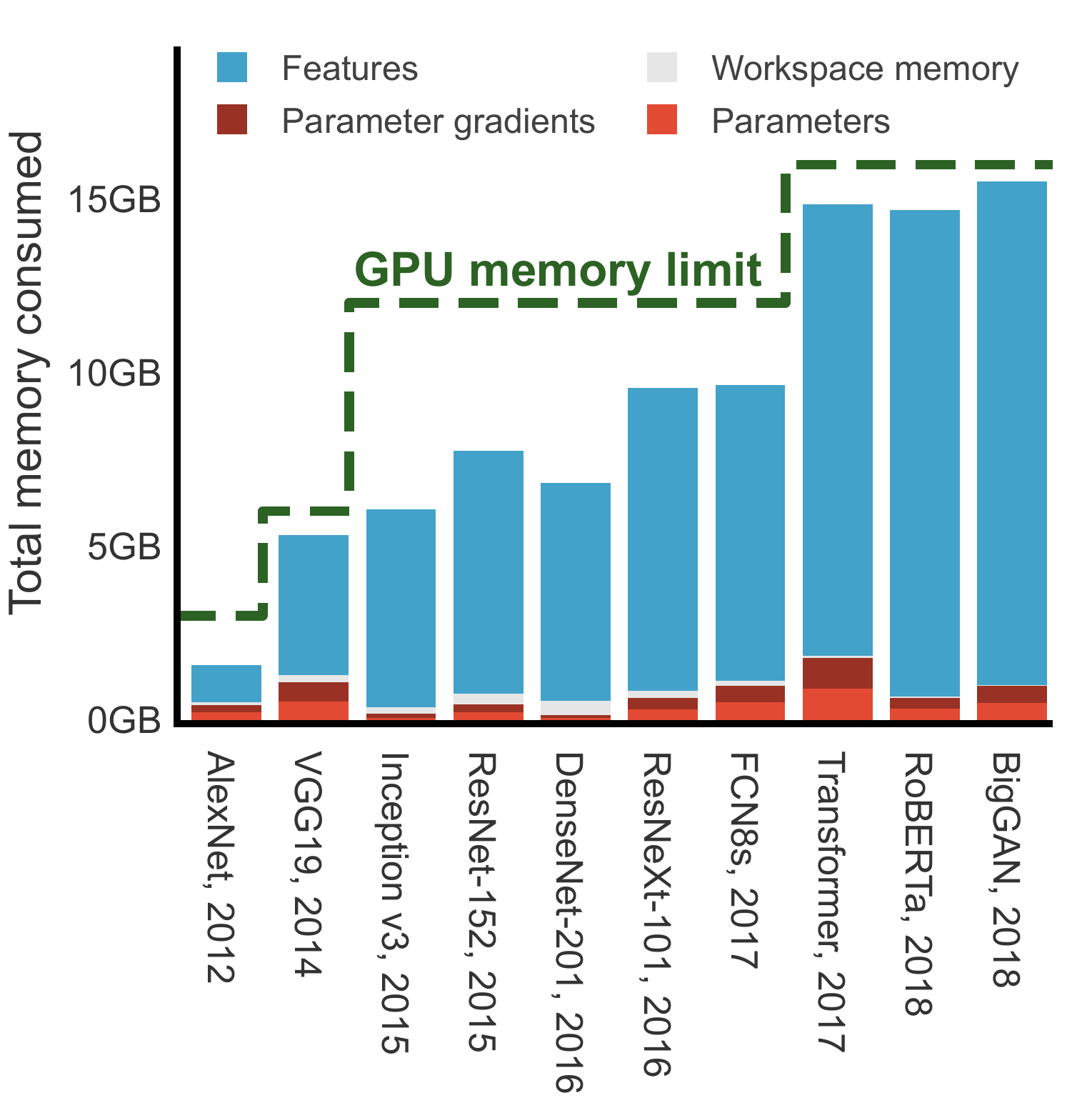}
	\vspace{-1.5em}
	\caption{Memory consumed by activations far outweigh parameters for popular model architectures. Moreover, advances in GPU DRAM capacity are quickly utilized by researchers; the dashed line notes the memory limit of the GPU used to train each model.\nocite{krizhevsky_imagenet_2012,simonyan_very_2014,szegedy_going_2015,he_deep_2016,huang_densely_2017,xie2017aggregated,long2015fully,vaswani_attention_2017,liu_roberta:_2019,brock2018large}}
	\label{fig:mem_breakdown}
\end{figure}

Subject to these constraints, we find solutions that minimize the amount of time it takes to perform a single training iteration.
We project schedules into space and time, allowing us to cast the objective as a linear expression.
This problem can then be solved using off-the-shelf mixed integer linear program solvers such as \citet{noauthor_glpk_nodate} or COIN-OR Branch-and-Cut \cite{forrest_coin-or_2019}. An optimal solution to the MILP will minimize the amount of additional compute cost within the memory budget.

\section{Related Work}
\label{sec:related}

\begin{table*}[tb]
	\footnotesize
	\begin{center}
		\newcolumntype{P}[1]{>{\centering\arraybackslash}p{#1}}
		\resizebox{\linewidth}{!}{\begin{tabular}{@{}m{3cm}m{8.5cm}P{1.4cm}P{1cm}P{1.4cm}@{}}
			\toprule
			\textsc{Method}                 & \textsc{Description}                                            & \textsc{General graphs} & \textsc{Cost aware} & \textsc{Memory aware} \\ \midrule
			Checkpoint all {\small (Ideal)} & No rematerialization. Default in deep learning frameworks.      & $\surd$                 & $\times$            & $\times$              \\
			Griewank et al. $\log{n}$       & \citet{griewank_algorithm_2000} \textsc{revolve} procedure      & $\times$                & $\times$            & $\times$              \\
			Chen et al. $\sqrt{n}$          & \citet{chen_training_2016} checkpointing heuristic              & $\times$                & $\times$            & $\times$              \\
			Chen et al. greedy              & \citet{chen_training_2016}, with search over parameter $b$      & $\times$                & $\times$            & $\sim$                \\ \midrule
			AP $\sqrt{n}$                   & Chen et al. $\sqrt{n}$ on articulation points + optimal R solve & $\sim$                  & $\times$            & $\times$              \\
			AP greedy                       & Chen et al. greedy on articulation points + optimal R solve     & $\sim$                  & $\times$            & $\sim$                \\
			Linearized $\sqrt{n}$           & Chen et al. $\sqrt{n}$ on topological sort + optimal R solve    & $\surd$                 & $\times$            & $\times$              \\
			Linearized greedy               & Chen et al. greedy on topological sort + optimal R solve        & $\surd$                 & $\times$            & $\sim$                \\ \midrule
			Checkmate ILP                    & Our ILP as formulated in Section~\ref{sec:ilp}                     & $\surd$                 & $\surd$             & $\surd$               \\
			Checkmate approx. & Our LP rounding approximation algorithm (Section~\ref{sec:approx}) & $\surd$ & $\surd$ & $\surd$ \\ \bottomrule
		\end{tabular}}
	\end{center}
	\caption{Rematerialization baselines and our extensions to make them applicable to non-linear architectures}
	\label{tab:baselines}
	\vskip -0.1in
\end{table*}

We categorize related work as checkpointing, reversible networks, distributed computation, and activation compression.

\textbf{Checkpointing and rematerialization}
~\citet{chen_training_2016} propose a heuristic for checkpointing idealized unit-cost linear $n$-layer graphs with $O(\sqrt{n})$ memory usage. \citet{griewank_algorithm_2000} checkpoint similar linear unit-cost graphs with $O(\log{n})$ memory usage and prove optimality for linear graphs with unit per-node cost and memory. In practice, DNN layers vary significantly in memory usage and computational cost \cite{sze2017efficient}, so these heuristics are not optimal in practice. 
\citet{chen_training_2016} also develop a greedy algorithm that checkpoints layers of a network in roughly memory equal segments, with a hyperparameter $b$ for the size of such segments. Still, neither procedure is cost-aware nor deallocates checkpoints when possible. \citet{gruslys_memory-efficient_2016} develop a dynamic programming algorithm for checkpoint selection in unrolled recurrent neural network training, exploiting their linear forward graphs. \citet{feng_cutting_2018} provide a dynamic program to select checkpoints that partition branching networks, but ignore layer costs and memory usage. \citet{doi:10.1080/10556788.2018.1459621} develop a divide-and-conquer strategy in programs. \citet{beaumont:hal-02131552} use dynamic programming for checkpoint selection in a specific architecture with joining sub-networks.

Intermediate value recomputation is also common in register allocation. Compiler backends lower an intermediate representation of code to an architecture-specific executable binary. During lowering, an abstract \textit{static single assignment} (SSA) graph of values and operations \cite{rosen_global_1988, cytron_efficiently_1991} is concretized by mapping values to a finite number of registers. If insufficient registers are available for an SSA form computation graph, values are \textit{spilled} to main memory by storing and later loading the value. Register allocation has been formulated as graph coloring problem \cite{chaitin_register_1981}, integer program \cite{goodwin_optimal_1996, lozano_combinatorial_2018}, and network flow~\cite{koes_global_2006}.

Register allocators may recompute constants and values with register-resident dependencies if the cost of doing so is less than the cost of a spill \cite{chaitin_register_1981, briggs_rematerialization_1992, punjani_register_2004}. While similar to our setup, register rematerialization is limited to exceptional values that can be recomputed in a single instruction with dependencies already in registers. For example, memory offset computations can be cheaply recomputed, and loads of constants can be statically resolved. In contrast, \OURS{} can recompute entire subgraphs of the program's data-flow.

During the evaluation of a single kernel, GPUs spill per-thread registers to a thread-local region of global memory (\ie{} local memory) \cite{micikevicius_local_2011, nvidia_nvidia_2017}. NN training executes DAGs of kernels and stores intermediate values in shared global memory. This produces a high range of value sizes, from 4 byte floats to gigabyte tensors, whereas CPU and GPU registers range from 1 to 64 bytes. Our problem of interkernel memory scheduling thus differs in scale from the classical problem of register allocation within a kernel or program. Rematerialization is more appropriate than copying values out of core as the cost of spilling values from global GPU memory to main memory (RAM) is substantial~\cite{micikevicius_local_2011, jain_gist:_2018}, though possible \cite{meng_training_2017}.

\textbf{Reversible Networks} ~\citet{gomez_reversible_2017} propose a reversible (approximately invertible) residual DNN architecture, where intermediate temporary values can be recomputed from values derived \textit{later} in the standard forward computation. Reversibility enables recomputation during the backward pass. \citet{bulo_-place_2018} replace only ReLU and batch normalization layers with invertible variants and reduce memory usage up to 50\%. We find rematerialization enables greater savings and a wider range of budgets, but reversibility is a promising complementary approach.

\textbf{Distributed computation} ~Orthogonal approaches to address the limited memory
problem are distributed-memory computations and gradient accumulation.
However, model parallelism requires access to additional expensive compute accelerators, fast networks, and non-trivial partitioning of model state to balance communication and computation~\cite{gholami2018integrated,jiaDataModelParallelism,mccandlishEmpiricalModelLargeBatch2018}.
Gradient accumulation enables larger batch sizes by computing the gradients with multiple sub-batches across a mini-batch. However, gradient accumulation can degrade performance as batch normalization performs poorly on small batch sizes~\cite{wuGroupNormalization2018,ioffe_batch_2015}.

\textbf{Activation compression} ~In some DNN applications, it is possible to process compressed representations with minimal accuracy loss. \citet{gueguen_faster_2018} classify discrete cosine transforms of JPEG images rather than raw images.
\citet{jain_gist:_2018} quantize activations, cutting memory usage in half.
Compression reduces memory usage by a constant factor, but reduces accuracy.
Our approach is mathematically equivalent to rematerialization-free training and incurs no accuracy penalty.

\section{Optimal Rematerialization}
\label{sec:ilp}

In this section, we develop an optimal solver that schedules computation and garbage collection during the evaluation of general data-flow graphs including those used in neural network training. Our proposed scheduler minimizes computation or execution time while guaranteeing that the schedule will not exceed device memory limitations. The rematerialization problem is formulated as a mixed integer linear program (MILP) that can be solved with standard commercial or open-source solvers.
 
\subsection{Problem definition}
\label{sec:ilp_problem}
A computation or data-flow graph ${G=(V,E)}$ is a directed acyclic graph with $n$ nodes ${V=\{v_1,\ldots,v_n\}}$ that represent operations yielding values (\eg{} tensors). 
Edges represent dependencies between operators, such as layer inputs in a neural network.
Nodes are numbered according to a topological order, such that operation $v_j$ may only depend on the results of operations $v_{i < j}$.

\pagebreak[2]
Each operator's output takes $M_v$ memory to store and costs $C_v$ to compute from its inputs.  We wish to find the terminal node $v_n$ with peak memory consumption under a memory budget, $\Mbudget$, and minimum total cost of computation.

\begin{figure*}[t]
	\centering
	\includegraphics[width=0.6\linewidth]{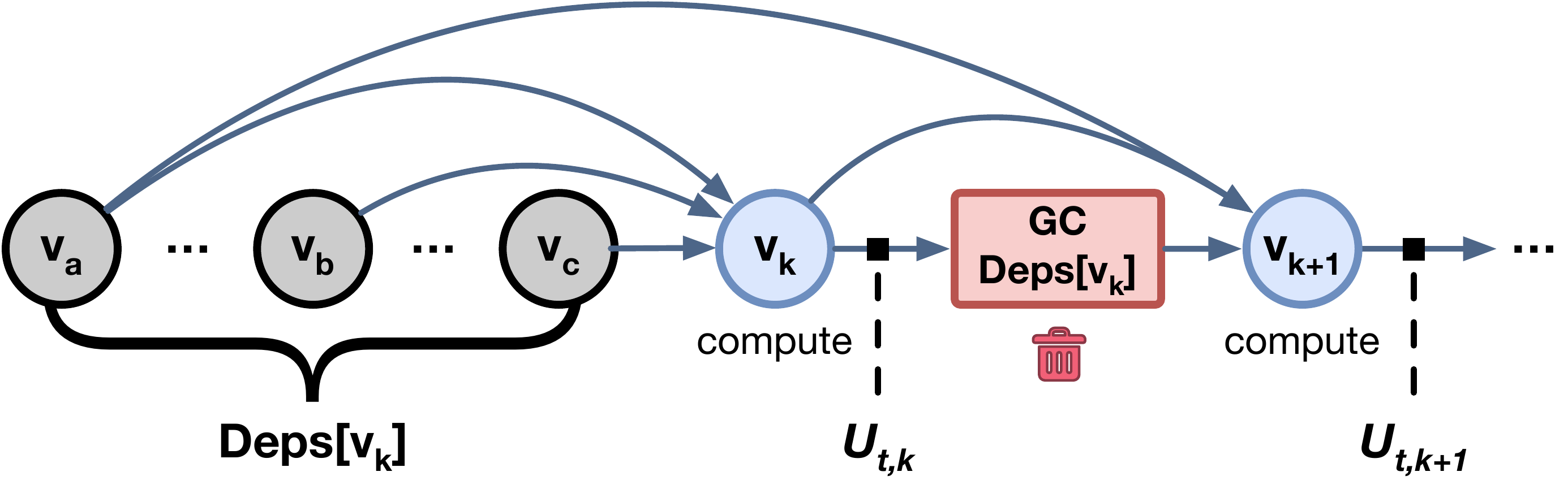}
	\caption{Dependencies of $v_k$ can only be garbage collected after it is evaluated. $U_{t,k}$ measures the memory used after evaluating $v_k$ and before deallocating its dependencies. $v_b$ and $v_c$ may be deallocated during garbage collection, but $v_a$ may not due to a forward edge.}
	\label{fig:gc_ordering}
\end{figure*}

\subsection{Representing a schedule}
\label{sec:ilp_partitioning}
We represent a schedule as a series of nodes being saved or (re)computed. We unroll the execution of the network into T stages and only allow a node to be computed once per stage.  $S_{t, i} \in \{0, 1\}$ indicates that the result of operation $i$ should be retained in memory at stage $t-1$ until stage $t$.
We also define $R_{t, i} \in \{0, 1\}$ be a binary variable reflecting whether operation i is recomputed at time step t.

Our representation generalizes checkpointing \cite{griewank_algorithm_2000, chen_training_2016, gruslys_memory-efficient_2016, siskind_divide-and-conquer_2018, feng_cutting_2018}, as values can be retained and deallocated many times, but comes at the cost of $O(Tn)$ decision variables. 

To trade-off the number of decision variables and schedule flexibility, we limit $T$
to $T=n$.
This allows for $O(n^2)$ operations and constant memory in linear graphs.

\subsection{Scheduling with ample memory}
First, consider neural network evaluation on a processor with ample memory. Even without a memory constraint, our solver must ensure that checkpointed and computed operations have dependencies resident in memory. Minimizing the total cost of computation across stages with dependency constraints yields objective (\ref{eq:objective1}):
\begin{argmini!}|l|[1]
    {R,S}{\sum_{t=1}^n \sum_{i=1}^t C_i R_{t,i}\label{eq:objective1}}
    {}{}
    \addConstraint{R_{t, j}}{\leq R_{t, i} + S_{t, i}\label{constr:compute_residency}}{\forall t ~\forall (v_i, v_j) \in E}
    \addConstraint{S_{t, i}}{\leq R_{t-1, i} + S_{t-1, i}\label{constr:cache_residency}\qquad}{\forall t \geq 2 ~\forall i}
    \addConstraint{\textstyle\sum_{i} S_{1, i}}{= 0\label{constr:noinitialcheckpoints}}{}
    \addConstraint{\textstyle\sum_{t} R_{t,n}}{\geq 1\label{constr:cover_last}}{}
    \addConstraint{R_{t, i}, S_{t, i} }{\in \{0, 1\}\label{constr:binary}}{\forall t ~\forall i}
\end{argmini!}
Constraints ensure feasibility and completion. 
Constraint (\ref{constr:compute_residency}) and (\ref{constr:cache_residency}) ensure that an operation is computed in stage $t$ only if all dependencies are available. 
To cover the edge case of the first stage, constraint (\ref{constr:noinitialcheckpoints}) specifies that no values are initially in memory. Finally, covering constraint (\ref{constr:cover_last}) ensures that the last node in the topological order is computed at some point in the schedule so that training progresses.

\subsection{Constraining memory utilization}
\label{sec:ilp_memory}

To constrain memory usage, we introduce memory accounting variables ${U_{t,k} \in \mathbb{R}_+}$ into the ILP. 
Let $U_{t,k}$ denote the memory used just after computing node $v_{k}$ in stage $t$. $U_{t,k}$ is defined recursively in terms of auxiliary binary variables $\Collectable{t,i,k}$ for $(v_i,v_k) \in E$, which specifies whether node $v_i$ may be deallocated in stage $t$ after evaluating node $v_k$.

We assume that (1) network inputs and parameters are always resident in memory and (2) enough space is allocated for gradients of the loss with respect to parameters.\footnote{While gradients can be deleted after updating parameters, we reserve constant space since many parameter optimizers such as SGD with momentum maintain gradient statistics.}
Parameter gradients are typically small, the same size as the parameters themselves. Additionally, at the beginning of a stage, all checkpointed values are resident in memory. Hence, we initialize the recurrence,
\begin{equation}
    U_{t,0} = \underbrace{\Min + 2\Mparam}_{\text{Constant overhead}} + \sum_{i=1}^n \underbrace{M_i S_{t,i}}_{\text{Checkpoints}}
    \label{constr:U_init}
\end{equation}
Suppose $U_{t,k}$ bytes of memory are in use after evaluating $v_k$. Before evaluating $v_{k+1}$, $v_k$ and dependencies (parents) of $v_k$ may be deallocated if there are no future uses. Then, an output tensor for the result of $v_{k+1}$ is allocated, consuming memory $M_{k+1}$. The timeline is depicted in Figure~{\ref{fig:gc_ordering}}, yielding recurrence (\ref{constr:mem_recurrence}):
\begin{equation}
    U_{t,k+1} = U_{t,k} - \GC{v_k}{t} + R_{t,k+1} M_{k+1},\label{constr:mem_recurrence}
\end{equation}
where $\GC{v_k}{t}$ is the amount of memory freed by deallocating $v_k$ and its parents at stage $t$. Let \begin{align*}
    \Deps{k} &= \{i : (v_i,v_k) \in E\}, \text{ and}\\
    \Users{i} &= \{j : (v_i,v_j) \in E\}
\end{align*} denote parents and children of a node, respectively. Then, in terms of auxiliary variable $\Collectable{t,i,k}$,
for $(v_i,v_k) \in E$,
\begin{align}
    &\GC{v_k}{t} = \textstyle\sum_{\substack{i \in \Deps{k}\\\cup \{k\}}} M_i * \Collectable{t,i,k},\label{eq:GC} ~\text{and} \\
    &\Collectable{t,i,k} = R_{t,k} * \underbrace{(1 - S_{t+1,i})}_{\text{Not checkpoint}} \prod_{\substack{j \in \Users{i}\\j>k}} \underbrace{(1 - R_{t,j})}_{\text{Not dep.}}
    \label{eq:collectable}
\end{align}
The second factor in \eqref{eq:collectable} ensures that $M_i$ bytes are freed only if $v_i$ is not checkpointed
for the next stage. The final factors ensure that $\Collectable{t,i,k}=0$ if any child of $v_i$ is computed in the stage, since then $v_i$ needs to be retained for later use.
Multiplying by $R_{t,k}$ in (\ref{eq:collectable}) ensures that values are only freed at most once per stage according to Theorem \ref{lemma:nodoublefree},
\begin{restatable}[No double deallocation]{theorem}{nodoublefree}
If \eqref{eq:collectable} holds for all $(v_i,v_k) \in E$, then $\sum_{k \in \Users{i}} \Collectable{t,i,k} \leq 1 ~\forall t, i$.\label{lemma:nodoublefree}
\end{restatable}
\begin{proof}
Assume for the sake of contradiction that $\exists k_1, k_2 \in \Users{i}$ such that $\Collectable{t,i,k_1} = \Collectable{t,i,k_2} = 1$. By the first factor in \eqref{eq:collectable}, we must have $R_{t,k_1} = R_{t,k_2} = 1$. Assume without loss of generality that $k_2 > k_1$. By the final factor in \eqref{eq:collectable}, we have $\Collectable{t,i,k_1} \leq 1 - R_{t,k_2} = 0$, which is a contradiction.
\end{proof}

\subsection{Linear reformulation of memory constraint}

While the recurrence (\ref{constr:U_init}-\ref{constr:mem_recurrence}) defining $U$ is linear, the right hand size of (\ref{eq:collectable}) is a polynomial. To express $\Collectable{}$ in our ILP, it must be defined via linear constraints. 
We rely on Lemma~\ref{lemma:linearproduct} and \ref{lemma:linearindicator} to reformulate (\ref{eq:collectable}) into a tractable form.
\begin{restatable}[Linear Reformulation of Binary Polynomial]{lemma}{linearproduct}
If $x_1,\ldots,x_n \in \{0, 1\}$, then
\begin{equation*}
    \prod_{i=1}^n x_i = \begin{cases}
    1 & \sum_{i=1}^n (1-x_i) = 0\\
    0 & \text{otherwise}
    \end{cases}
\end{equation*}\label{lemma:linearproduct}
\end{restatable}
\begin{proof}
If all $x_1, \ldots, x_n = 1$, then $\sum_{i=1}^n (1 - x_i) = 0$ and we have $\Pi_{i=1}^n x_i = 1$. If otherwise any $x_j = 0$, then we have $\Pi_{i=1}^n x_i = 0$, as desired.
This can also be seen as an application of De Morgan's laws for boolean arithmetic.
\end{proof}

\begin{restatable}[Linear Reformulation of Indicator Constraints]{lemma}{linearindicator}
Given $0 \leq y \leq \kappa$ where $y$ is integral and $\kappa$ is a constant upper bound on $y$, then
\begin{align*}
    x = \begin{cases}
    1 & y = 0\\
    0 & \text{otherwise}
    \end{cases}
\end{align*}
if and only if $x \in \{0, 1\}$ and $(1 - x) \leq y \leq \kappa (1 - x)$.
\label{lemma:linearindicator}
\end{restatable}
\begin{proof}
For the forward direction, first note that by construction, $x \in \{0, 1\}$. If $y = 0$ and $x = 1$, then 
$(1 - x) = 0 \leq y \leq 0 = \kappa (1 - x)$. Similarly, if $y \geq 1$ and $x = 0$, then %we also have 
$1 \leq y \leq \kappa$, which is true since %we assume
$0 \leq y \leq \kappa$ and $y$ is integral. The converse holds similarly.
\end{proof}

To reformulate Constraint~\ref{eq:collectable}, let $\NumHazards{t,i,k}$ be the number of zero factors on the RHS of the constraint. This is a linear function of the decision variables,
$$\NumHazards{t,i,k} = (1 - R_{t,k}) + S_{t+1,i} + \sum_{\substack{j \in \Users{i}\\j>k}} R_{t,j}$$
Applying Lemma \ref{lemma:linearproduct} to the polynomial constraint, we have,
\begin{align}
    \Collectable{t,i,k} = \begin{cases}
    1 & \NumHazards{t,i,k} = 0 \\
    0 & \text{otherwise}
    \end{cases}
    \label{eq:freeindicator}
\end{align}
By Lemma \ref{lemma:linearindicator}, if $\kappa$ is the maximum value that $\NumHazards{t,i,k}$ can assume, the following constraints are equivalent to (\ref{eq:freeindicator}),
\begin{subequations}
\begin{align}
    \Collectable{t,i,k} &\in \{0, 1\}\label{constr:cble_binary}\\
    1-\Collectable{t,i,k} &\leq \NumHazards{t,i,k}\label{constr:cble_upper}\\
    \kappa(1-\Collectable{t,i,k}) &\geq \NumHazards{t,i,k}\label{constr:cble_lower}
\end{align}
\end{subequations}

\subsection{Tractability via frontier-advancing stages}
Fixing the execution order of nodes in the graph can improve the running time of the algorithm. In eager-execution frameworks such as PyTorch, the order is given by user code and operations are executed serially. Separating ordering and allocation is common in compiler design, and both LLVM \cite{lattner_llvm:_2002} and GCC \cite{olesen_register_2011} have separate instruction scheduling and register allocation passes.

Any topological order of the nodes is a possible execution order. 
Given a topological order, such as the one introduced in Section~\ref{sec:ilp_problem}, we partition the schedule into frontier-advancing stages such that node $v_i$ is evaluated for the first time in stage $i$. We replace constraints (\ref{constr:noinitialcheckpoints}, \ref{constr:cover_last}) that ensure the last node is computed with stricter constraints (\ref{constr:frontier_adv}-\ref{constr:Rlowtri}),
\begin{subequations}\begin{align}
R_{i,i} &= 1 ~~\forall i ~~~\text{(frontier-advancing partitions)}\label{constr:frontier_adv}\\
\textstyle\sum_{i \geq t} S_{t,i} &= 0 ~~~\text{(lower tri., no initial checkpoints)} \label{constr:Slowtri}\\
\textstyle\sum_{i > t} R_{t,i} &= 0 ~~~\text{(lower triangular)} \label{constr:Rlowtri}
\end{align}\end{subequations}
This reduces the feasible set, constraining the search space and improving running time.
For an 8 layer ($n=17$) linear graph neural network with unit $C_i, M_i$ at a memory budget of 4, Gurobi optimizes the unpartitioned MILP in 9.4 hours and the partitioned MILP in 0.23 seconds to the same objective. In Appendix~\ref{sec:integrality_gap}, we analyze the integrality gap of both forms of the problem to understand the speedup.

\subsection{Complete Integer Linear Program formulation}
\label{sec:ilp:complete_ilp}
The complete memory constrained MILP follows in (\ref{eq:objective2}), with $O(|V||E|)$ variables and constraints. 
\begin{argmini}|l|[0]
	{R,S,U,\Collectable{}}{\sum_{t=1}^n \sum_{i=1}^t C_i R_{t,i}\label{eq:objective2}}
	{}{}
    \addConstraint{(\ref{constr:compute_residency}), (\ref{constr:cache_residency}), (\ref{constr:binary}), (\ref{constr:U_init}), (\ref{constr:mem_recurrence})}{}
    \addConstraint{(\ref{constr:cble_binary}), (\ref{constr:cble_upper}), (\ref{constr:cble_lower}), (\ref{constr:frontier_adv}), \eqref{constr:Slowtri}, \eqref{constr:Rlowtri}}{}
	\addConstraint{U_{t,k} \leq \Mbudget}{}{}
\end{argmini}

\subsection{Constraints implied by optimality}
\label{sec:ilp:implied_constraints}

Problem~\ref{eq:objective2} can be simplified by removing constraints implied by optimality of a solution. 
${\Collectable{t,k,k} = 1}$ only if operation $k$ is spuriously evaluated with no uses of the result. Hence, the solver can set $R_{t,k}=0$ to reduce cost. We eliminate $|V|^2$ variables $\Collectable{t,k,k}$, assumed to be 0, by modifying (\ref{eq:GC}) to only sum over $i \in \Deps{k}$. These variables can be computed inexpensively after solving.

\subsection{Generating an execution plan}
\label{sec:ilp:execution_plan}

Given a feasible solution to (\ref{eq:objective2}), $(R, S, U, \Collectable{})$, Algorithm~\ref{alg:execution_plan} generates an execution plan
via a row major scan of $R$ and $S$ with deallocations determined by $\Collectable{}$. 
An execution plan is a program $P = (\com{1},\ldots,\com{k})$ with $k$ statements. When statement $\texttt{\%r = compute v}$ 
is interpreted, operation $v$ is evaluated. The symbol $\texttt{\%r}$ denotes a virtual register used to track the resulting value. Statement $\texttt{deallocate \%r}$ marks the value tracked by virtual register $\texttt{\%r}$ for garbage collection.

The execution plan generated by Algorithm~\ref{alg:execution_plan} is further optimized by moving deallocations earlier in the plan when possible. Spurious checkpoints that are unused in a stage can be deallocated at the start of the stage rather than during the stage. Still, this code motion is unnecessary for feasibility as the solver guarantees that the unoptimized schedule will not exceed the desired memory budget.

The execution plan can either be interpreted during training, or encoded as a static computation graph. In this work, we generate a static graph $G'=(V', E')$ from the plan, which is executed by a numerical machine learning framework. See Section~\ref{sec:eval:methodology} for implementation details.
\begin{algorithm}[t]
    \caption{Generate execution plan}
    \label{alg:execution_plan}
\begin{algorithmic}
    \STATE {\bfseries Input:} graph $G=(V, E)$, feasible $(R, S, \Collectable{})$
    \STATE {\bfseries Output:} execution plan $P=(\com{1}, \ldots, \com{k})$
    \STATE Initialize $\Regs{1\ldots |V|} = -1$, $r = 0$, $P = ()$.
    \FOR{$t=1$ {\bfseries to} $|V|$}
        \FOR{$k=1$ {\bfseries to} $|V|$}
            \IF{$R_{t,k}$}
            	\STATE // \textit{Materialize $v_k$}
                \STATE add \texttt{\%$r$ = compute  v\textsubscript{k}} to $P$
                \STATE $\Regs{k} = r$
                \STATE $r = r + 1$
            \ENDIF
	    	\STATE // \textit{Free $v_k$ and dependencies}
            \FOR{$i \in \Deps{k} \cup \{k\}$}
                \IF{$\Collectable{t, i, k}$}
                    \STATE add \texttt{deallocate \%$\Regs{i}$} to $P$
                \ENDIF
            \ENDFOR
        \ENDFOR
    \ENDFOR
    \STATE \textbf{return} $P$
\end{algorithmic}
\end{algorithm}

\subsection{Cost model}
\label{sec:ilp:cost_model}
To estimate the runtime of a training iteration under a rematerialization plan, we apply an additive cost model (\ref{eq:objective1}), incurring cost $C_i$ when node $v_i$ is evaluated. Costs are determined prior to MILP construction by profiling network layers on target hardware with random inputs across a range of batch sizes and input shapes, and exclude static graph construction and input generation time. 
As neural network operations consist of dense numerical kernels such as matrix multiplication, these runtimes are low variance and largely independent of the specific input data \cite{jia_exploring_2018, sivathanu_astra:_2019}. However, forward pass time per batch item decreases with increasing batch size due to improved data parallelism \cite{canziani_analysis_2016}, so it is important to compute costs with appropriate input dimensions.

The memory consumption of each value in the data-flow graph is computed statically as input and output sizes are known. 
Values are dense, multi-dimensional tensors stored at 4 byte floating point precision. The computed consumption $M_i$ is used to construct memory constraints (\ref{constr:U_init}-\ref{constr:mem_recurrence}).

\section{Approximation}
\label{sec:approx}
Many of our benchmark problem instances are tractable to solve using off-the-shelf integer linear program solvers, with practical solve times ranging from seconds to an hour. ILP results in this paper
are obtained with a 1 hour time limit on a computer with at least 24 cores. 
Relative to training time, \eg{} 21 days for the BERT model~\cite{devlin_bert:_2018}, solving the ILP adds less than a percent of runtime overhead.

While COTS solvers such as COIN-OR~\cite{forrest_coin-or_2019} leverage methods like branch-and-bound to aggressively prune the decision space, they can take superpolynomial time in the worst-case and solving ILPs is NP-hard in general.
In the worst-case, for neural network architectures with hundreds of layers, it is not feasible to solve the rematerialization problem via our ILP. An instance of the VGG16 architecture~\cite{simonyan_very_2014} takes seconds to solve. For DenseNet161~\cite{huang_densely_2017}, no feasible solution was found within one day. 

For many classical NP-hard problems, approximation algorithms give solutions close to optimal with polynomial runtime. We review a linear program that produces fractional solutions in polynomial time in Section~\ref{sec:approx:relaxlp}. 
Using the fractional solutions, we present a two-phase rounding algorithm in Section~\ref{sec:approx:partial_rounding} that rounds a subset of the decision variables, then finds a minimum cost, feasible setting of the remaining variables to find near-optimal integral solutions. 

\subsection{Relaxing integrality constraints}
\label{sec:approx:relaxlp}
By relaxing integrality constraints (\ref{constr:binary}), the problem becomes trivial to solve as it is a linear program over continuous variables. It is well known that an LP is solvable in polynomial time via Karmarkar's algorithm~\cite{karmarkar1984new} or barrier methods~\cite{nesterov1994interior}. With relaxation $R,S, \Collectable{} \in [0, 1]$,
the objective (\ref{eq:objective1}) defines a lower-bound for the cost of the optimal integral solution. 

Rounding is a common approach to find approximate integral solutions given the result of an LP relaxation. For example, one can achieve a $\frac{3}{4}$-approximation for MAX SAT~\cite{yannakakis1994approximation} via a simple combination of randomized rounding ($\mathbf{Pr} \left[ x^{\text{int}}_{i} = 1 \right] = x_i^*$) and deterministic rounding ($x^{\text{int}}_{i} = 1$ if $x^*_i \geq p$, where commonly $p=0.5$). 

We attempt to round the fractional solution $R^*, S^*$ using these two strategies, and then apply Algorithm~\ref{alg:execution_plan} to $R^\text{int}, S^\text{int}$.
However, direct application of deterministic rounding returns infeasible results: the rounded solution violates constraints.
Randomized rounding may show more promise as a single relaxed solution can be used to sample many integral solutions, some of which are hopefully feasible. Unfortunately, using randomized rounding with the LP relaxation for VGG16 at a 4$\times$ smaller budget than default, we could not find a single feasible solution out of 50,000 samples.

\begin{algorithm}[t!]
    \caption{Two-phase rounding}
    \label{alg:two_phase_rounding}
\begin{algorithmic}
    \STATE {\bfseries Input:} Fractional checkpoint matrix $S^*$ from LP
    \STATE {\bfseries Output:} Binary $S^\text{int}$, $R^\text{int}$, $\Collectable{}$
    \STATE Round $S^*$ deterministically: $S_{t,i}^\text{int} \gets \mathbbm{1}[S_{t,i}^* > 0.5]$
    \STATE $R^\text{int} \gets \mathbf{I}_n$ thereby satisfying (\ref{constr:frontier_adv})
    \WHILE{$\exists t\geq 2, i\in[n]$ such that $S^\text{int}_{t,i} > R^\text{int}_{t-1,i} + S^\text{int}_{t-1,i}$ \ie{} \eqref{constr:cache_residency} violated}
        \STATE Compute $v_i$ to materialize checkpoint: $R^\text{int}_{t-1,i} \gets 1$
    \ENDWHILE
    \WHILE{$\exists t\geq 1, (i,j) \in E$ such that $R^\text{int}_{t,j} > R^\text{int}_{t,i} + S^\text{int}_{t,i}$\\ \ie{} \eqref{constr:compute_residency} violated}
        \STATE Compute $v_i$ as temporary for dependency: $R^\text{int}_{t,i} \gets 1$
    \ENDWHILE
    \STATE Evaluate $\Collectable{}$ by simulating execution
    \STATE \textbf{return} $S^\text{int}$, $R^\text{int}$, $\Collectable{}$
\end{algorithmic}
\end{algorithm}

\subsection{A two-phase rounding strategy}
\label{sec:approx:partial_rounding}

To find feasible solutions, we introduce \textit{two-phase rounding}, detailed in Algorithm~\ref{alg:two_phase_rounding}. Two-phase rounding is applicable when a subset of variables can be solved in polynomial time given the remaining variables. Our approximation algorithm only rounds the checkpoint matrix $S^*$. Given $S^*$, we solve for the conditionally optimal binary computation matrix $R^\text{int}$ by setting as few values to $1$ as possible. 
Algorithm~\ref{alg:two_phase_rounding} begins with an all-zero matrix $R^\text{int} = 0$, then iteratively corrects violated correctness constraints.

Note that during any of the above steps, once we set some $R_{i,j}^\text{int} = 1$, the variable is never changed. Algorithm~\ref{alg:two_phase_rounding} corrects constraints in a particular order so that constraints that are satisfied will continue to be satisfied as other violated constraints are corrected. The matrix $R^\text{int}$ generated by this rounding scheme will be optimal up to the choice of $S^\text{int}$ as every entry in $R^\text{int}$ is set to $1$ if and only if it is necessary to satisfy a constraint. In implementation, we detect and correct violations of (\ref{constr:compute_residency}) in reverse topological order for each stage, scanning $R^\text{int},S^\text{int}$ matrices from right to left.

\subsection{Memory budget feasibility}
Since we approximate $S$ by rounding the fractional solution, $S^\text{int}, R^\text{int}$ can be infeasible by the budget constraint $U_{t,k} \leq \Mbudget$. While the fractional solution may come under the budget and two-phase rounding preserves correctness constraints, the rounding procedure makes no attempt to maintain budget feasibility. Therefore, we leave an allowance on the total memory budget constraint ($U_{t, k} \leq (1 - \epsilon) \Mbudget$). We empirically find $\epsilon = 0.1$ to work well.

%%%%%%%%%%%%%%%%%%%%%%%%%%%%%%%%%%%%%%%%%%%%%%%%%%%%%%%%%%%%%%%%%%%%%%%%%%%%%%%%%%%
\begin{figure*}[t]
	\centering
	\includegraphics[width=\textwidth]{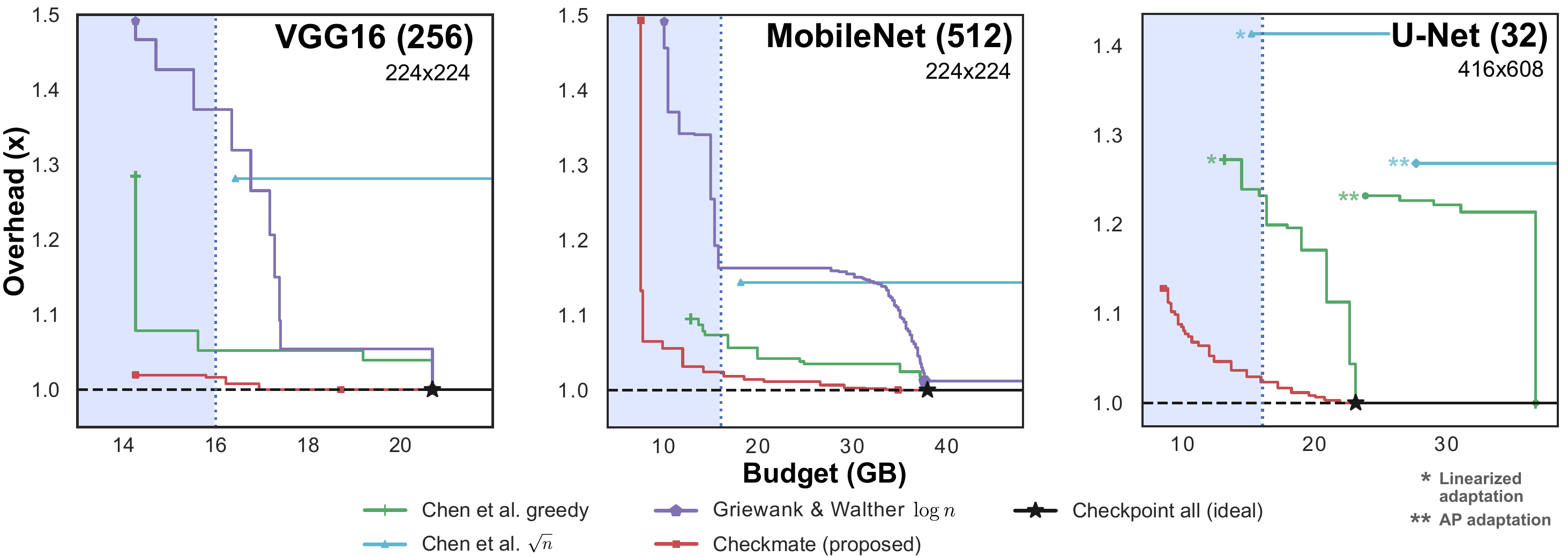}
	\vspace{-1em}
	\caption{Computational overhead versus memory budget for (a) VGG16 image classification NN \cite{simonyan_very_2014}, (b) MobileNet image classification NN, and (c) the U-Net semantic segmentation NN \cite{ronneberger_u-net:_2015}. Overhead is with respect to the best possible strategy without a memory restriction based on a profile-based cost model of a single NVIDIA V100 GPU. For U-Net (c), at the 16 GB V100 memory budget, we achieve a $1.20\times$ speedup over the best baseline---linearized greedy---and a $1.38\times$ speedup over the next best---linearized $\sqrt{n}$. \textbf{Takeaway:} our model- and hardware-aware solver produces in-budget solutions with the lowest overhead on linear networks (a-b), and dramatically lowers memory consumption \textit{and} overhead on complex architectures (c).}
	\label{fig:budget_sweep}
\end{figure*}
%%%%%%%%%%%%%%%%%%%%%%%%%%%%%%%%%%%%%%%%%%%%%%%%%%%%%%%%%%%%%%%%%%%%%%%%%%%%%%%%%%%

\section{Evaluation}
\label{sec:eval}

In this section, we investigate the impact of tensor rematerialization on the cost and memory usage of DNN training. We study the following experimental questions: (1) \textit{What is the trade-off between memory usage and computational overhead when using rematerialization?} (2) \textit{Are large inputs practical with rematerialization?} and (3) \textit{How well can we approximate the optimal rematerialization policy?}

We compare our proposed solver against baseline heuristics on representative image classification and high resolution semantic segmentation models including VGG16, VGG19, ResNet50, MobileNet, U-Net and FCN with VGG layers, and SegNet. 
As prior work is largely limited to 
linear graphs, we propose novel extensions where necessary for comparison. 
Results show that optimal rematerialization allows significantly lower computational overhead than baselines at all memory budgets, and lower memory usage than previously possible. As a consequence, optimal rematerialization allows training with larger input sizes than previously possible, up to 5.1$\times$ higher batch sizes on the same accelerator. Finally, we find that our two-phase rounding approximation algorithm finds near-optimal solutions in polynomial time.

\subsection{Baselines and generalizations}
\label{sec:baselines}

Table~\ref{tab:baselines} summarizes baseline rematerialization strategies.
The nominal evaluation strategy stores all features generated during the
forward pass for use during the backward pass---this is the default in frameworks such as TensorFlow. Hence, every layer is computed once.
We refer to this baseline as \textit{Checkpoint all}, an ideal approach given ample memory.

On the linear graph architectures, such as VGG16 and MobileNet (v1), we directly apply prior work from 
\citet{griewank_algorithm_2000} and \citet{chen_training_2016}, baselines referred to as \textit{Griewank and Walther $\log{n}$}, \textit{Chen et al. $\sqrt{n}$} and \textit{Chen et al. greedy}. To build a tradeoff curve for computation versus memory budget, we search over the segment size hyperparameter $b$ in the greedy strategy. 
However, these baselines cannot be 
used for modern architectures with residual connections. For a fair comparison, we extend the $\sqrt{n}$ and greedy algorithms to apply to general computation graphs (\eg{} ResNet50 and U-Net).

\citet{chen_training_2016} suggests manually annotating good checkpointing candidates in a computation graph. For the first extensions, denoted by \textit{AP $\sqrt{n}$} and \textit{AP greedy}, we automatically identify \textit{articulation points}, or \textit{cut vertices}, vertices that disconnect the forward pass DAG, and use these as candidates. The heuristics then select a subset of these candidates, and we work backwards from the checkpoints to identify which nodes require recomputation.

Still, some networks have few articulation points, including U-Net. We also extend heuristics by treating the original graph as a linear network, with nodes connected in topological order, again backing out the minimal recomputations from the selected checkpoints. These extensions are referred to as \textit{Linearized $\sqrt{n}$} and \textit{Linearized greedy}.

Sections~\ref{sec:eval:AP} and \ref{sec:eval:linearized} provide more details on our generalizations. Note that all proposed generalizations exactly reproduce the original heuristics on linear networks.

\subsection{Evaluation setup}
\label{sec:eval:methodology}

\OURS{} is implemented in Tensorflow 2.0~\cite{abadi_tensorflow:_2016}, accepting user-defined models 
expressed via the high-level Keras interface. 
We extract the forward and backward computation graph, then construct and solve optimization problem~(\ref{eq:objective2}) with the Gurobi mathematical
programming library as an integer linear program.
Finally, \OURS{} translates solutions into execution plans and constructs a new static training graph.
Together, these components form the \OURS{} system, illustrated in Figure~\ref{fig:system}.

To accelerate problem construction, decision variables $R$ and $S$ are expressed as lower triangular matrices, as are accounting variables $U$. $\Collectable{}$ is represented as a $|V|\times{}|E|$ matrix. 
Except for our maximum batch size experiments, solutions are generated with a user-configurable time limit of $3600$ seconds, though the majority of problems solve within minutes. Problems with exceptionally large batch sizes or heavily constrained memory budgets may reach this time limit while the solver attempts to prove that the problem is infeasible. The cost of a solution is measured with a profile-based cost model (Section~\ref{sec:ilp:cost_model}) and compared to the ideal, unachievable cost with no recomputation.

The feasible set of our optimal ILP formulation is a superset of baseline heuristics. We implement baselines as a static policy for the decision variable $S$ and then solve for the lowest-cost recomputation schedule using a similar procedure to that described in Algorithm~\ref{alg:two_phase_rounding}.

\subsection{What is the trade-off between memory usage and computational overhead?}
\label{sec:eval:tradeoff}
Figure \ref{fig:budget_sweep} compares remateralization strategies on 
VGG-16, MobileNet, and U-Net. The y-axis shows the 
computational overhead of checkpointing in terms of time as compared to
baseline. The time is computed by profiling each individual layer of the
network. The x-axis shows the total memory budget required to run each model
with the specified batch size, computed for single precision training. 
Except for the $\sqrt{n}$ heuristics, each rematerialization algorithm has a knob to trade-off the
amount of recomputation and memory usage, where a smaller memory budget leads to higher overhead.

\textbf{Takeaways:} For all three DNNs,
\OURS{} produces clearly faster execution plans as compared 
to algorithms proposed by \citet{chen_training_2016} and \citet{griewank_algorithm_2000} -- over $1.2\times$
faster than the next best on U-Net at the NVIDIA V100 memory budget. Our framework allows training a U-Net
at a batch size of 32 images per GPU with less than 10\% higher overhead. This would require 23 GB of memory without
rematerialization, or with the original baselines without our generalizations.

\subsection{Are large inputs practical with rematerialization?}

%%%%%%%%%%%%%%%%%%%%%%%%%%%%%%%%%%%%%%%%%%%%%%%%%%%%%%%%%%%%%%%%%%%%%%%%%%%%
\begin{figure}[t]
	\centering
	\includegraphics[width=\linewidth]{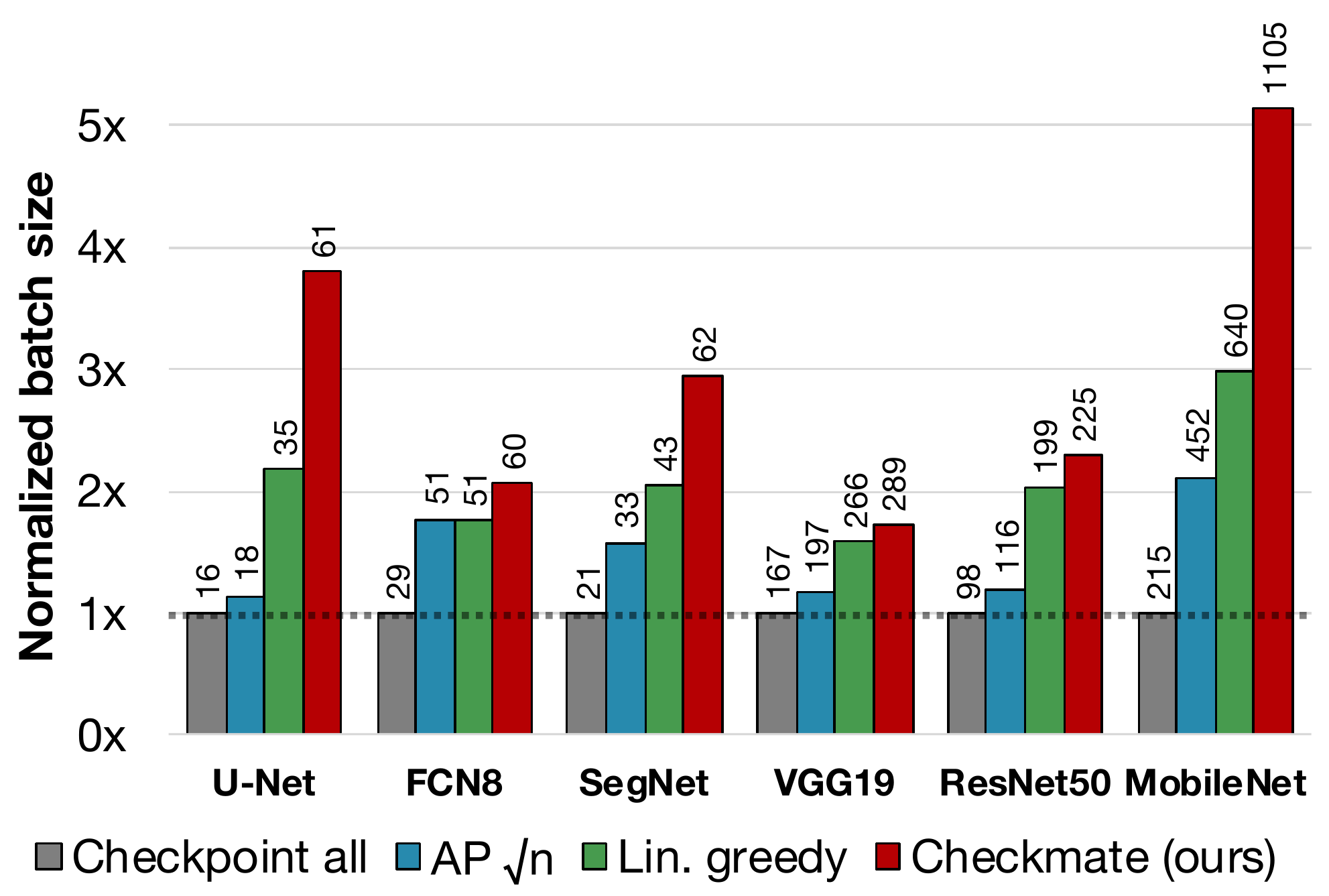}
	\vspace{-1em}
	\caption{
		Maximum batch size possible on a single NVIDIA V100 GPU when using different generalized rematerialization strategies with at most a single extra forward pass. We enable increasing batch size by up to \emph{$5.1\times$} over the current practice of caching all activations (on MobileNet), and up to $1.73\times$ over the best checkpointing scheme (on U-Net).}
	\label{fig:m}
\end{figure}
%%%%%%%%%%%%%%%%%%%%%%%%%%%%%%%%%%%%%%%%%%%%%%%%%%%%%%%%%%%%%%%%%%%%%%%%%%%%

The maximum batch size enabled by different rematerialization strategies 
is shown in Figure~\ref{fig:m}. The y-axis shows the theoretical maximum batch size we could feasibly train with bounded compute cost. 
This is calculated by enforcing that the total cost must be less than the cost of 
performing just one additional forward pass. 
That is, in Figure~\ref{fig:m} the cost is at most an additional forward pass higher, \emph{if} the specified batch size would have fit in GPU memory. To find Checkmate's maximum batch size, we reformulate Problem (\ref{eq:objective2}) to maximize a batch size variable $B \in \mathbb{N}$ subject to modified memory constraints that use $B*M_i$ in place of $M_i$ and subject to an additional cost constraint,

\begin{equation}
	\sum_{t=1}^n \sum_{i=1}^t C_i R_{t,i} \leq 2 \sum_{v_i \in G_\text{fwd}} C_i + \sum_{v_i \in G_\text{bwd}} C_i. \label{eq:cost_constr}
\end{equation}

The modified integer program has quadratic constraints, and is difficult to solve. We set a time limit of one day for the experiment, but Gurobi may be unable to reach optimality within that limit. Figure~\ref{fig:m} then provides a lower bound on the maximum batch size that \OURS{} can achieve.

For fair comparison on the non-linear graphs used in U-Net, FCN, and ResNet, we use the AP $\sqrt{n}$ and linearized greedy baseline generalizations described in Section~\ref{sec:baselines}. 
For the baselines, we iterate over batch sizes, find candidate solutions (multiple candidates for linearized greedy), and filter out the solutions that cost more than an additional forward pass or that would exceed the 16GB memory budget. The iteration stops when no solutions are available.

Costs are measured in FLOPs, determined statically. U-Net, FCN8 and SegNet semantic segmentation networks use a resolution of $416\times 608$, and classification networks ResNet50, VGG19 and MobileNet use resolution $224\times 224$.

\textbf{Takeaways:} We can theoretically increase the batch size of U-Net to $61$ at a high resolution, an unprecedented result.
For many tasks such as semantic segmentation, where U-Net is commonly used, it is not possible to use
batch sizes greater than $16$, depending on resolution. This is sub-optimal for batch normalization layers, and
being able to increase the batch size by $3.8\times$ ($61$ vs $16$ at this resolution) is quite significant.
Orthogonal approaches to achieve this
include model parallelism and distributed memory batch normalization which can be
significantly more difficult to implement and have high communication costs.

Furthermore, for MobileNet, \OURS{} allows a batch size of $1105$ which
is $1.73\times$ higher than the best baseline solution, a greedy heuristic, and $5.1\times$ common practice, checkpointing all activations.
The same schedules can also be used to increase image resolution rather than batch size.

\subsection{How well can we approximate the optimal rematerialization policy?}
\label{sec:eval:approx}

To understand how well our LP rounding strategy (Section~\ref{sec:approx}) approximates the ILP, we measure the ratio $\textsc{Cost}_\text{approx} / \textsc{Cost}_\text{opt}$, \ie{} the speedup of the optimal schedule, in FLOPs. As in Section~\ref{sec:eval:tradeoff}, we solve each strategy at a range of memory budgets, then compute the geometric mean of the ratio across budgets. The aggregated ratio is used because some budgets are feasible via the ILP but not via the approximations. Table~\ref{table:approx} shows results. 
The two-phase deterministic rounding approach has approximation factors close to optimal, at most $1.06\times$ for all tested architectures.

%%%%%%%%%%%%%%%%%%%%%%%%%%%%%%%%%%%%%%%%%%%%%%%%%%%%%%%%%%%%%%%%%%%%%%%%%%%%
\setlength{\tabcolsep}{4pt}
\begin{table}[t]
\centering
\resizebox{\linewidth}{!}{\begin{tabular}{r|cccc} 
\toprule
\multicolumn{1}{l}{} & \multicolumn{1}{c}{\begin{tabular}[c]{@{}c@{}}AP\\$\sqrt{n}$\end{tabular}} & \multicolumn{1}{c}{\begin{tabular}[c]{@{}c@{}}AP\\greedy\end{tabular}} & \multicolumn{1}{c}{\begin{tabular}[c]{@{}c@{}}Griewank\\$\log n$\end{tabular}} & \multicolumn{1}{c}{\begin{tabular}[c]{@{}c@{}}Two-phase\\LP rounding\end{tabular}} \\ 
\midrule
MobileNet &   1.14$\times$ & 1.07$\times$ & 7.07$\times$ & \textbf{1.06$\times$} \\
VGG16 &       1.28$\times$ & 1.06$\times$ & 1.44$\times$ & \textbf{1.01$\times$} \\
VGG19 &       1.54$\times$ & 1.39$\times$ & 1.75$\times$ & \textbf{1.00$\times$} \\
U-Net & 1.27$\times$ & 1.23$\times$ & -          &\textbf{1.03$\times$} \\
ResNet50 &    1.20$\times$ & 1.25$\times$ & -          &\textbf{1.05$\times$} \\
\bottomrule
\end{tabular}}
\caption{Approximation ratios for baseline heuristics and our LP rounding strategy. Results are given as the geometric mean speedup of the optimal ILP across feasible budgets. \label{table:approx}}
\end{table}
%%%%%%%%%%%%%%%%%%%%%%%%%%%%%%%%%%%%%%%%%%%%%%%%%%%%%%%%%%%%%%%%%%%%%%%%%%%%

\section{Conclusion}
\label{sec:conclusion}
One of the main challenges when training large neural networks is the limited
capacity of high-bandwidth memory on accelerators such as GPUs and TPUs. 
This has created a memory wall that limits the size of the models that can be trained.
The bottleneck for state-of-the-art model development is now memory 
rather than data and compute availability,
and we expect this trend to worsen in the future.

To address this challenge,
we proposed a novel rematerialization algorithm which allows large models to be
trained with limited available memory.
Our method does not make the strong assumptions required in prior work, supporting general
non-linear computation graphs such as residual networks and capturing
the impact of non-uniform memory usage and computation cost throughout the graph
with a hardware-aware, profile-guided cost model. 
We presented an ILP formulation for the problem, implemented the \OURS{} system 
for optimal rematerialization in TensorFlow, and tested the proposed 
system on a range of 
neural network models. In evaluation, we find that optimal rematerialization has minimal computational overhead at a wide range of memory budgets and showed 
that \OURS{} enables practitioners to train high-resolution models with significantly larger batch sizes.
Finally, a novel two-phase rounding strategy closely approximates the optimal solver.

\clearpage
\section*{Acknowledgements}
\label{sec:acks}
We would like to thank Barna Saha and Laurent El Ghaoui for guidance on approximation, Mong H. Ng for help in evaluation, and the paper and artifact reviewers for helpful suggestions. In addition to NSF CISE Expeditions Award CCF-1730628 and ONR PECASE N000141612723, this work was supported by gifts from Alibaba, Amazon Web Services, Ant Financial, CapitalOne, Ericsson, Facebook, Futurewei, Google, Intel, Microsoft, NVIDIA, Scotiabank, Splunk and VMware. This work was also supported by the NSF GRFP under Grant No. DGE-1752814. Any opinions, findings, and conclusions or recommendations expressed in this material are those of the author(s) and do not necessarily reflect the views of the NSF.

\bibliography{references}
\bibliographystyle{mlsys2020}

\appendix

\begin{figure*}[t]
    \centering
    \includegraphics[width=\textwidth]{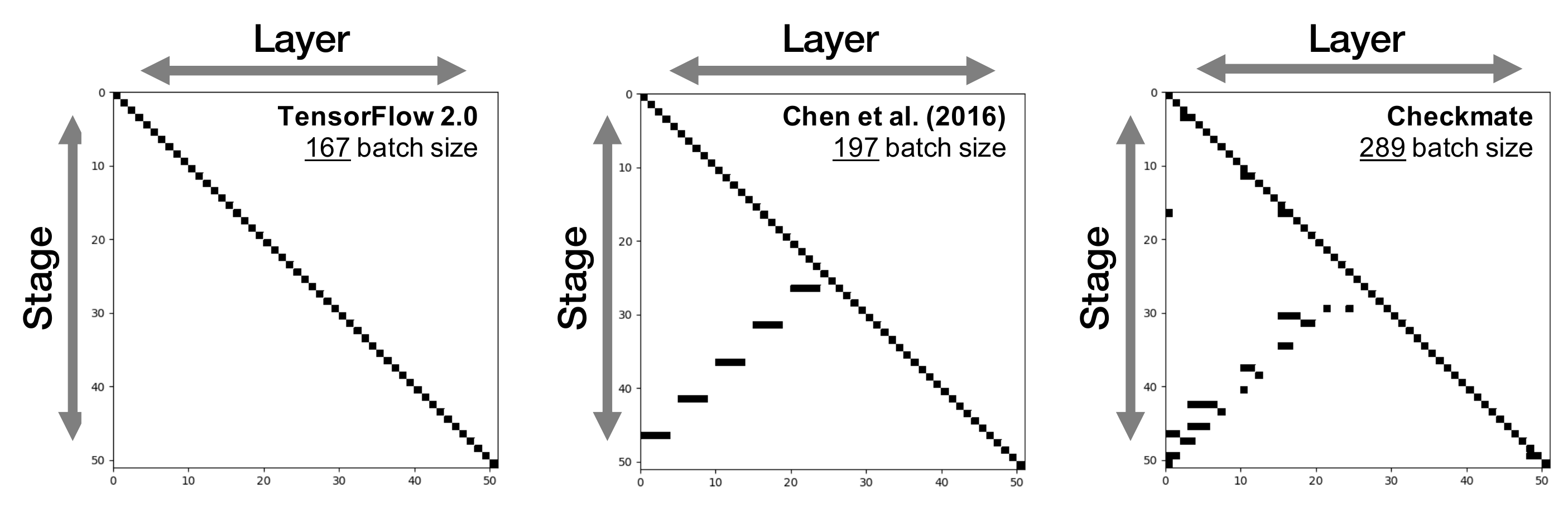}
    \caption{We visualize schedules ($R$ matrix) for VGG19. The $R$ matrix denotes when each layer in a neural network is evaluated. Rows refer to execution stage while columns refer to a particular layer. We make no distinction between forward and backward nodes. TensorFlow 2.0 can train VGG19 on a V100 with a batch size of 167. By applying the heuristic from \citet{chen_training_2016}, a V100 can sustain training at a batch size od 197. However, our proposed method in \OURS{} sustains training at a batch size of 289, representing a 73\% improvement. The integer linear program took 10 seconds to solve to optimality.} 
    \label{fig:sched_viz_segnet}
\end{figure*}

\section{Integrality gap}
\label{sec:integrality_gap}

To understand why the partitioned variant of the MILP (Section~\ref{sec:ilp_partitioning}) is faster to solve via branch-and-bound, we can measure the integrality gap for particular problem instances. The integrality gap is the maximum ratio between the optimal value of the ILP and its relaxation, defined as follows:
\begin{align*}
    IG = \max_{I} \frac{\textsc{Cost}_{int}}{\textsc{Cost}_{frac}},
\end{align*}
where $\textsc{Cost}_{int}$ and $\textsc{Cost}_{frac}$ are the optimal value the ILP and that of its relaxation, respectively. $I$ describes a problem instance, $$I = (G, C, M, \Mbudget).$$ As our ILP is a minimization problem, $\textsc{Cost}_{int} \geq \textsc{Cost}_{frac}$ for all $I$, and $IG \geq 1$. 
While it is not possible to measure the ratio between the ILP and LP solutions for all problem instances, the ratio for any particular problem instance gives a lower bound on the integrality gap.

For the 8-layer linear neural network graph discussed in Section~\ref{sec:ilp_partitioning}, frontier-advancement reduces the integrality gap from 21.56 to 1.18, \ie{} the LP relaxation is significantly tighter. In branch-and-bound algorithms for ILP optimiztion, a subset of feasible solutions can be pruned if the LP relaxation over the subset yields an objective higher than the best integer solution found thus far. With a tight LP relaxation, this condition for pruning is often met, so fewer solutions need to be enumerated.

\section{Generalizations of prior work}
\subsection{AP $\sqrt{n}$ and AP greedy}\label{sec:eval:AP} We identify Articulation Points (AP) in the undirected form of 
the forward pass data-flow graph as candidates for checkpointing. APs are vertices that 
increase the number of connected components (\ie{} disconnect) the graph if removed, and can be 
identified in time $O(V + E)$ via a modified DFS traversal \cite{holder_graph_2008}.
An articulation 
point $v_a$ is a good candidate for checkpointing as subsequent vertices in the topological order have 
no dependencies on vertices before $v_a$ in the order. DNN computation graphs are connected, so each intermediate tensor can be reconstructed 
from a single articulation point earlier in the topological order, or the input if there is no such AP. APs include the input and 
output nodes of residual blocks in ResNet, but not vertices inside blocks. 
We apply Chen's heuristics to checkpoint a subset of these candidates, then solve for the optimal recomputation plan $R$ to restore correctness. Solving for $R$ ensures that a node's dependencies are resident prior to evaluation.

We could find $R$ by solving the optimization problem~(\ref{eq:objective2}) with additional constraints on $S$ that encode the heuristically selected checkpoints. However, as $S$ is given, the optimization is solvable in $O(|V||E|)$ via a graph traversal per row of $R$ that fills in entries when a needed value is not in memory by the same process described in Section~\ref{sec:approx:partial_rounding}.

\subsection{Linearized $\sqrt{n}$ and Linearized greedy}\label{sec:eval:linearized} The forward graph of the DNN ${G_\text{fwd} = (V_\text{fwd}, E_\text{fwd})}$ can be treated as a linear graph ${G_\text{lin} = (V_\text{fwd}, E_\text{lin})}$ with edges connecting consecutive vertices in a topological order: $$E_\text{lin} = \{(v_1, v_2), (v_2, v_3), \ldots, (v_{L-1}, v_L)\}$$ 
While $G_\text{lin}$ does not properly encode data dependencies, it is a linear graph that baselines can analyze. To extend a baseline, we apply it to $G_\text{lin}$, generate checkpoint matrix $S$ from the resulting checkpoint set, and find the optimal $R$ as with the AP baselines.

\section{Hardness of rematerialization}
\label{sec:hardness}

\citet{sethi_complete_1973} reduced 3-SAT to a decision problem based on register allocation in straight line programs, with no recomputation permitted. Such programs can be represented by result-rooted Directed Ayclic Graphs (DAGs), with nodes corresponding to operations and edges labeled by values. In Sethi's graphs, the desired results are the roots of the DAG. If a program has no common subexpressions, \ie{} the graph forms a tree, optimal allocation is possible via a linear time tree traversal \cite{nakata_compiling_1967}. However, Sethi's reduction shows a register allocation decision problem in the general case---whether a result-rooted DAG can be computed with fewer than $k$ registers without recomputation---is NP-complete.

\balance{}
The decision problem characterizes computation of a DAG as a sequence of four possible moves of stones, or registers, on the nodes of the graph, analogous to statements discussed in Section~\ref{sec:ilp:execution_plan}. 
The valid moves are to (1) place a register at a leaf, computing it, or (2) pick up a register from a node. Also, if there are registers at all children of a node $x$, then it is valid to (3) place a register at $x$, computing it, or (4) move a stone to $x$ from one of the children of $x$, computing $x$. The register allocation problem reduces to the following no-overhead rematerialization decision problem (RP-DEC):
\begin{definition}
(RP-DEC): Given result-terminated data-flow DAG $G = (V, E)$ corresponding to a program, with unit cost to compute each node and unit memory for the results of each node, does there exist an execution plan that evaluates the leaf (terminal) node $t \in V$ with maximum memory usage $b$ at cost at most $|V|$?
\end{definition}
RP-DEC is decidable by solving the memory-constrained form of Problem 1 with sufficient stages, then checking if the returned execution plan has cost at most $|V|$.
RP-DEC closely resembles Sethi's decision problem, differing only in subtleties. The register allocation DAG is rooted at the desired result $t$ whereas a data-flow graph terminates at the result. Second, register-based computations can be in place, \eg{} a summation $a + b$ may be written to the same location as either of the operands. In neural network computation graphs, we cannot perform all computations in place, so we did not make this assumption. To reduce Sethi's decision problem to RP-DEC, given result-rooted DAG $G$, construct result-terminated $G'$ by reversing all edges. Then, if Sethi's instance allows for at most $k$ registers, allow for a memory budget of $b = k + 1$ bytes: one byte to temporarily write outputs of operations that would have been written in place.

Despite hardness of register allocation, \citet{goodwin_optimal_1996} observe that a 0-1 integer program for optimal allocation under an instruction schedule has empirical complexity $O(n^{2.5})$, polynomial in the number of constraints. Similarly, 
Section~\ref{sec:eval} shows that the frontier-advancing, constrained optimization problem (\ref{eq:objective2}) is tractable for many networks.

\section{Comparison of approximations}

In Section~\ref{sec:approx}, we discussed an approximation strategy based on rounding the LP relaxation, evaluated with deterministic rounding in Section~\ref{sec:eval:approx}. Figure~\ref{fig:approx_rounding} compares schedules produced by our proposed two-phase rounding strategy when the $S^*$ matrix from the LP relaxation is rounded with a randomized and a deterministic approach. While two-phase randomized rounding of $S^*$ offers a range of feasible solutions, two-phase deterministic rounding produces consistently lower cost schedules. 
While appropriate for VGG16, for MobileNet, our budget allowance $\epsilon=0.1$ is overly conservative as schedules use less memory than the 16 GB budget. A search procedure over $\epsilon \in [0, 1]$ could be used to produce more efficient schedules.

\begin{figure}[t]
    \centering
    \includegraphics[width=\columnwidth]{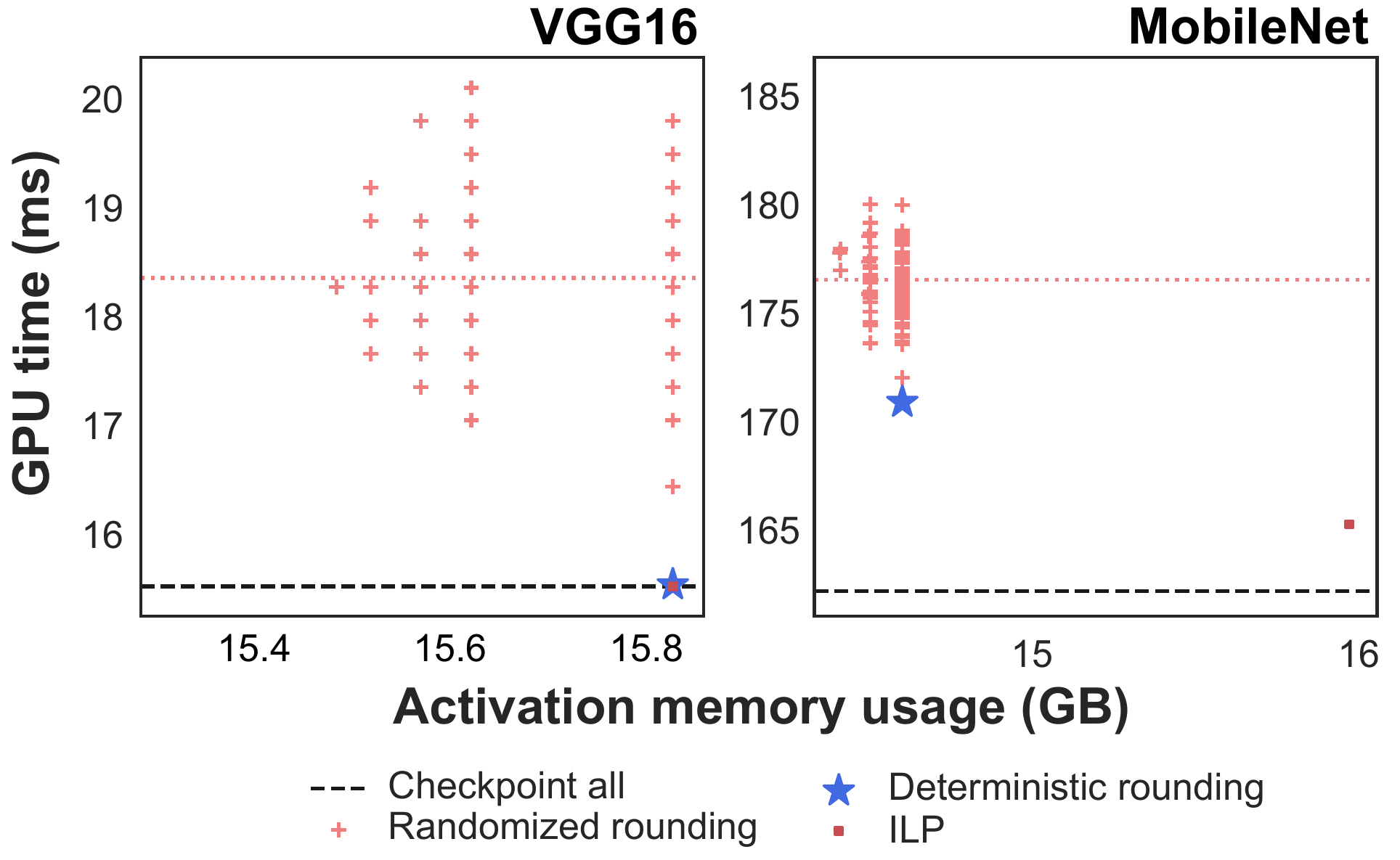}
    \vspace{-1em}
    \caption{Comparison of the two-phase LP rounding approximation with randomized rounding of $S^*$ and deterministic rounding of $S^*$ on different models. We compare memory usage and computational cost (objective), in milliseconds according to profile-based cost model. The average of the randomized rounding costs is shown as a dotted line.}
    \label{fig:approx_rounding}
\end{figure}

\end{document}